\documentclass{article}

\usepackage[nonatbib,final,preprint]{nips_2018}

\usepackage[utf8]{inputenc}
\usepackage[natbib,style=authoryear,maxbibnames=50, maxcitenames=2, doi=false,
isbn=false,url=false]{biblatex}
\usepackage[standard]{ntheorem}
\usepackage{stmaryrd}
\usepackage{amssymb}
\usepackage{mathtools}
\mathtoolsset{showonlyrefs}  
\usepackage{booktabs}
\usepackage{etoolbox}
\usepackage{siunitx}
\usepackage{caption}
\usepackage{subcaption}
\usepackage{xspace}
\usepackage{nicefrac}
\usepackage{hyperref}
\usepackage{url}

\newcommand{\sharpP}{\#P\xspace}
\newcommand{\NP}{NP\xspace}

\newcommand{\fromto}{\longrightarrow}
\renewcommand{\to}{\fromto}
\renewcommand{\vec}[1]{\boldsymbol{#1}}

\newcommand{\on}{\operatorname}

\newcommand{\IR}{\mathbb{R}}
\newcommand{\IS}{\mathbb{S}}
\newcommand{\IN}{\mathbb{N}}

\newcommand{\hyp}{\on{hyp}}

\newcommand*{\sothat}{\colon}
\newcommand*{\from}{\colon}

\newcommand{\argmin}{\operatorname*{\arg\,\min}}

\newcommand{\argsort}{\operatorname*{\arg\, sort}}
\newcommand{\cX}{\mathcal{X}}
\newcommand{\cQ}{\mathcal{Q}}
\newcommand{\cR}{\mathcal{R}}
\newcommand{\cZ}{\mathcal{Z}}

\newcommand{\rhobc}{\rho_{\text{FETA}}}
\newcommand{\newnet}{\textsc{FATE-Net}\xspace}
\newcommand{\bordanet}{\textsc{FETA-Net}\xspace}
\newcommand{\ranksvm}{\textsc{RankSVM}\xspace}
\newcommand{\rankboost}{\textsc{RankBoost}\xspace}

\newcommand{\err}{\textsc{ERR}\xspace}

\newcommand{\ranknet}{\textsc{RankNet}\xspace}
\newcommand{\listnet}{\textsc{ListNet}\xspace}
\DeclarePairedDelimiter{\indic}{\llbracket}{\rrbracket}
\DeclarePairedDelimiterX{\norm}[1]{\lVert}{\rVert}{#1}

\newcommand{\makename}[3][s]{%
  \expandafter\newcommand\csname #2\endcsname{#3\xspace}%
  \expandafter\newcommand\csname #2s\endcsname{#3#1\xspace}%
}
\makename{problem}{ranking task}
\makename{context}{context}
\makename{cotext}{cotext}

\usepackage{acronym}
\newacro{FETA}{First Evaluate Then Aggregate}
\newacro{FATE}{First Aggregate Then Evaluate}

\newif\ifcameraready
\camerareadytrue

\author{
  Karlson Pfannschmidt\\
  Department of Computer Science\\
  Warburger Str. 100\\
  Paderborn University\\
  Germany\\
  \texttt{kiudee@mail.upb.de}\\
  \And
  Pritha Gupta \\
  Department of Computer Science\\
  Warburger Str. 100\\
  Paderborn University\\
  Germany\\
  \texttt{prithag@mail.upb.de}\\
  \And
  Eyke Hüllermeier \\
  Department of Computer Science\\
  Warburger Str. 100\\
  Paderborn University\\
  Germany\\
  \texttt{eyke@upb.de}\\
}

\begin{document}

\title{Deep Architectures for Learning \\ Context-dependent Ranking Functions}

\maketitle

\begin{abstract}
%
Object ranking is an important problem in the realm of preference
learning. On the basis of training data in the form of a set of rankings
of objects, which are typically represented as feature vectors, the goal is to learn a ranking function that predicts a linear order of any new set of objects.
%
Current approaches commonly focus on ranking by scoring, i.e., on learning an underlying latent utility function that seeks to capture the inherent utility of each object. These approaches, however, are not able to take possible effects of \emph{context-dependence} into account, where context-dependence means that the utility or usefulness of an object may also depend on what other objects are available as alternatives.  
%
In this paper, we formalize the problem of context-dependent ranking and present two  general approaches based on two natural representations of context-dependent ranking functions. Both approaches are instantiated by means of appropriate neural network architectures, which are evaluated on suitable benchmark task.
\end{abstract}


\section{Introduction}

In preference learning \citep{PL-book}, the learner is generally
%
provided with a set of items
(e.g., products) for which preferences are known, and the task is to learn a
function that predicts preferences for a new set of items
(e.g., new products not seen so far), or for the same set of items in a
different situation (e.g., the same products but for a different user).
Frequently, the predicted preference relation is required to form a total
order, in which case we also speak of a \emph{ranking problem}.
In fact, among the problems in the realm of preference learning, the task of
``learning to rank'' has probably received the most attention in the literature
so far, and a number of different ranking problems have already been introduced.

The focus of this paper is on so-called \emph{object ranking} \citep{cohe_lt98,kami_as10}. Given training data in the form of
a set of exemplary rankings of subsets of objects, the goal in object ranking is
to learn a ranking function that is able to predict the ranking of any new set
of objects.
As a typical example, consider an eCommerce scenario, in which a customer is
ranking a set of products, each characterized by different properties and
attributes, according to her preferences.

In economics, classical choice theory assumes that, for a given user, each alternative has an
inherent utility, and that choices and decisions are made on the basis of these
utilities. Yet, many studies have shown that these idealized assumptions are often violated
in practice.
For example, choices are also influenced by the decision context, i.e., by the availability of other alternatives \citep{huber1982adding,simonson1992,tversky1993,dhar2000}.
Motivated by observations of that kind, the focus of this paper is on the
problem of \emph{context-dependent ranking}.
In this regard, our contributions are as follows.
First, we formalize the problem of context-dependent ranking and present two
general approaches based on two natural representations of context-dependent
ranking functions: \ac{FETA} and \ac{FATE}.
Second, both approaches are instantiated by means of appropriate neural network
architectures, called \bordanet and \newnet, respectively.
These architectures can be trained in an end-to-end manner.
Third, we conduct an experimental evaluation of our methods, using both
synthetic and real-world data, for which context-dependence is playing a
relevant role.
Empirically, we are able to show that our methods outperform traditional
approaches on these tasks.

\section{Object ranking}
\label{sec:learning}

We assume a reference set of objects denoted by $\cX$, where each object
$\vec{x} \in \cX$ is described by a feature vector; thus, an object is a vector
$\vec{x} = (x_1, \dotsc, x_d) \in \IR^d$, and $\cX \subseteq \IR^d$.
A \emph{\problem} is specified by a finite subset $Q = \{\vec{x}_1, \ldots, \vec{x}_n \} \subseteq \cX$ of objects, for some $n \in \mathbb{N}$, and the task itself consists of predicting a preferential ordering of these objects, that is, a \emph{ranking}. 
The latter is encoded in terms of a permutation $\pi \in \mathbb{S}_n$, where $\mathbb{S}_n$ denotes the set of all permutations of length $n$, i.e.,
all mappings $[n] \fromto [n]$ (symmetric group of order $n$).
A permutation $\pi$ represents the total order $\succ$ such that $\vec{x}_{\pi^{-1}(k)} \succ \vec{x}_{\pi^{-1}(k+1)}$ for all $k \in [n-1]$, 
where $\pi(k)$ is the position of the $k$th object $\vec{x}_k$, and $\pi^{-1}(k)$ the index of the object on position $k$ ($\pi$ is often called a \emph{ranking} and $\pi^{-1}$ an \emph{ordering}).
Formally, a ranking function can thus be understood as a mapping $\rho: \, \cQ \fromto \cR$,
where $\cQ = 2^\mathcal{X} \setminus \emptyset$ is the \emph{\problem space} (or simply task space) and $\cR = \bigcup_{n \in [N]} \mathbb{S}_n$ the \emph{ranking space}.

Methods for object ranking seek to induce a ranking function from training data $\{ (Q_i , \pi_i) \}_{i=1}^N$ in the form of exemplary ranking tasks $Q_i$ together with observed rankings $\pi_i$. Typically, this is done by learning a latent utility function
\begin{equation}\label{eq:uf} 
U \from \cX \fromto \IR \enspace ,
\end{equation}
which assigns a real-valued score to each object $\vec{x} \in \cX$.
Given a task $Q$, a ranking is then simply constructed by sorting the objects $\vec{x} \in Q$ according to their scores. This approach
implies important properties of the induced preferences, i.e.,
the set of rankings $\{\rho(Q) \, \vert \, Q \in \cQ \}$ produced by $\rho$ on the \problem space.
In particular, preferences have to be \emph{transitive} and, moreover,
\emph{\context-independent}.

Context-independence means that the preference between two items $\vec{x}$ and $\vec{y}$ does not depend on the set of other items $Q \setminus \{ \vec{x} , \vec{y} \}$ in the query (which we consider as defining the context in which $\vec{x}$ and $\vec{y}$ are ranked). More formally, consider any pair of items $\vec{x}, \vec{y} \in \mathcal{X}$ and any subsets $Q, Q' \subseteq \mathcal{X}$ such that $\vec{x}, \vec{y} \in Q \cap Q'$. Moreover, let $\succ$ be the ranking induced by $\rho$ on $Q$ and $\succ'$ the corresponding ranking on $Q'$. Then, context-independence implies that $\vec{x} \succ \vec{y}$ if an only if $\vec{x} \succ' \vec{y}$. Obviously, context-independence is closely connected to the famous Luce axiom of choice \citep{luce1959}.

Let us note that the notion of ``context'' is also used with a different meaning in the learning-to-rank literature (and in machine learning in general), namely as a kind of extra dimension. For instance, \citet{agrawal06} illustrate their notion of ``context-sensitive ranking'' with an example in which objects are actors and the extra dimension is the film genre: ``Contextual preferences take the form that item $i_1$ is preferred to item $i_2$ in the context of $X$. For example, a preference might state the choice for Nicole Kidman over Penelope Cruz in drama movies, whereas another preference might choose Penelope Cruz over Nicole Kidman in the context of Spanish dramas.'' Obviously, this differs from our definition of ``context'', which is derived from its use in the economics literature. 


\section{Context-Dependent Ranking}

In practice, the assumption of context-independence of preferences is often
violated, because 
preferences of individuals are influenced by the context in
which decisions are made \citep{bettman1998constructive}.
In economics, three major context effects have been identified in the literature:
the compromise effect \citep{simonson1989choice},
the attraction effect \citep{huber1983market}, and the similarity effect
\citep{tversky1972elimination}.
To capture effects of context-dependence, our goal is to learn a generalized latent utility function
\begin{equation}\label{eq:latut}
  U \from \cX \times 2^{\cX} \fromto \IR \enspace ,
\end{equation}
which can be used in the same way as (\ref{eq:uf}) to assign a score to each object of the \problem.
Since the utility function has a second argument, namely a context, it allows for representing \context-dependent ranking functions
$$
\rho \big(\{ \vec{x}_1, \ldots , \vec{x}_n \} \big) = 
\argsort_{i \in [n]} U \big(\vec{x}_i , C_i \big) \, ,
$$
where, for each object $\vec{x}_i$ in a task $Q = \{\vec{x}_1, \ldots , 
 \vec{x}_n \}$, we denote by $C_i = C(\vec{x}_i) = Q \setminus \{ \vec{x}_i \} = \{\vec{x}_1, \ldots , \vec{x}_{i-1} , \vec{x}_{i+1}, \ldots , \vec{x}_n \}$ its context in this task.

In this section, we present two general approaches based on two natural  representations of context-dependent ranking functions. These representations are based on two rather natural ways to decompose the problem of assigning a context-dependent score to an object: ``First evaluate then aggregate'' (\ac{FETA}) first evaluates the object in each ``sub-context'' of a fixed size, and then aggregates these evaluations, whereas ``first aggregate then evaluate'' (\ac{FATE}) first aggregates the entire set of alternatives into a single representative, and then evaluates the object in the context of that representative. Interestingly, the former approach has already been used in the literature \citep{volkovs2009}, at least implicitly, while the latter is novel to the best of our knowledge. Before explaining these approaches in more detail, we make a few more general remarks on the representation of context-dependent ranking functions.

\subsection{Modeling Context-Dependent Utility}

The representation of a context-dependent utility function \eqref{eq:latut} comes with (at least) two important challenges, which are both connected to the fact that the second argument of such a function is a \emph{set} of \emph{variable} size. First, the arity of the function is therefore not fixed, because different ranking tasks, and hence different contexts, can have different size. Second, the function should be \emph{permutation-invariant} (symmetric) with regard to the elements in the second argument, the context, because the order in which the alternative objects are presented does not play any role.  
Formally,  function $f \from \cX^k \fromto \IR$ is \emph{permutation-invariant} if and only if
$f(\vec{x}_1, \vec{x}_2, \dots, \vec{x}_k) =
    f \left(\vec{x}_{\pi(1)}, \vec{x}_{\pi(2)}, \dots, \vec{x}_{\pi(k)} \right)$
  for all permutations $\pi$ of the indices $[k]=\{1, \ldots , k\}$.
  A function with this property is also called
\emph{symmetric} \citep{stanley2001}.


As for the problem of rating objects in contexts of variable size, one possibility is to decompose a context into sub-contexts of a fixed size $k$. More specifically, the idea is to learn context-dependent utility functions of the form
$U_k \from \cX \times \cX^k \fromto \IR$,
and to represent the original function \eqref{eq:latut} as an aggregation
\begin{equation}
U(\vec{x} , C)  =  \sum_{k=1}^K  \bar{U}_k(\vec{x}, C) \\
  = \sum_{k=1}^K \frac{1}{{|C| \choose k}} \sum_{C' \subseteq C, |C'| = k} U_k(\vec{x} , C') \, .
\label{eq:agg}
\end{equation}
Note that, provided permutation-invariance holds for $U_k$ as well as the aggregation, $U$ itself will also be symmetric. Taking the arithmetic average as an aggregation function, the second condition is obviously satisfied. Thus, the problem that essentially remains is to guarantee the symmetry of $U_k$.

Roughly speaking, the idea of the above decomposition is that dependencies and interaction effects between objects only occur up to a certain order $K$, or at least can be limited to this order without loosing too much information. This is an assumption that is commonly made in the literature on aggregation
functions \citep{grab_af} and also in other types of applications. The special cases $k=0$ and $k=1$ correspond to independence and pairwise interaction, respectively.

An interesting question concerns the expressivity of a $K$th order approximation \eqref{eq:agg}, where, for example, expressivity could be measured in terms of the number of different ranking functions that can be defined on $\cX$.  
To study this question, suppose that $\cX$ is finite and consists of $N$ objects. 
Obviously, for $K=0$, only $N!$ different ranking functions can be produced,
because the entire function is determined by the order on the maximal \problem
$Q = \mathcal{X}$. Naturally, the number of possible ranking functions should increase with increasing $K$. 
For the extreme case $K=N-1$, we can indeed show that all ranking functions can be
generated (see Proposition~1 in the Appendix).




\subsection{\acl{FETA}}\label{ssub:FETA}


Our first approach realizes \eqref{eq:agg} for the special case $K=1$, which can be seen as a first-order approximation of a fully \context-dependent ranking function.
Thus, we propose the representation of a ranking function $\rho$
which, in addition to a utility function $U_0 \sothat \mathcal{X} \fromto [0,1]$, is based on a pairwise predicate $U_1 \sothat \mathcal{X} \times \mathcal{X} \fromto [0,1]$. Given a \problem
$Q = \{ \vec{x}_1 , \ldots , \vec{x}_n \} \subseteq \mathcal{X}$, a ranking is obtained as follows:
\begin{equation}
  \rho_{\text{\acs{FETA}}}(Q)  = 
  \argsort_{i \in [n]} \; U(\vec{x}_i, C_i) 
  =   \argsort_{i \in [n]}  \; \biggl\{
U_0(\vec{x}_i) + \frac{1}{n-1} \sum_{j \in [n] \setminus \{i\}} U_1(\vec{x}_i , \vec{x}_j)
\biggr\}\label{eq:bc}
\end{equation}
We refer to this approach as \acf{FETA}. 

The observation that \ac{FETA} is able to capture context-dependence is quite obvious.
As a simple illustration, suppose that $U_0 \equiv 0$ and $U_1$ is given on $\mathcal{X} = \{a,b,c,d\}$
as follows:
\[
  \big(U_1(x,y) \big)_{x,y \in \mathcal{X}} = \left( \begin{array}{cccc}
    -   & 0.7 & 0.5 & 0.1 \\
    0.2 & -   & 0.8 & 0.9 \\
    0.5 & 0.2 & -   & 0.4 \\
    0.7 & 0.1 & 0.5 & -
  \end{array} \right)
\]
For the queries $Q_1= \{a, b, c\}$ and $Q_2 = \{a, b, d\}$ we 
obtain rankings $\rho_{\text{\acs{FETA}}}(Q_1) = a \succ b \succ c$ and $\rho_{\text{\acs{FETA}}}(Q_2) = b \succ a \succ d$. That is, the preference between $a$ and $b$ changes depending on whether the third item to be ranked is $c$ or $d$.

It is important to note that our interpretation of $U_1$ is not the standard interpretation in terms of a pairwise preference relation. Specific properties such as asymmetry ($U_1(\vec{x}, \vec{y}) = 1 - U_1(\vec{y}, \vec{x})$) are therefore not necessarily required, although they could be incorporated for the purpose of regularization. Instead, $U_1(\vec{x}, \vec{y})$ should be interpreted more generally as a measure of support given by $\vec{y}$ to $\vec{x}$. This interpretation is in line with \citet{Ragain2016}, who model distributions on rankings using Markov chains. Here, individual preferences are defined in terms of probabilities (of the stationary distribution), and binary relations $U_1(\vec{x}, \vec{y})$ define transition probabilities. Thus, $U_1(\vec{x}, \vec{y})$ is the probability of moving from $\vec{y}$ to $\vec{x}$, and the larger the probability of being in $\vec{x}$, the higher the preference for this item.
Roughly speaking, $U_1(\vec{x}, \vec{y})$ is a measure of how favorable it is for $\vec{x}$ that $\vec{y}$ is part of its context $C$. In other words, a large value $U_1(\vec{x}, \vec{y})$ suggests that, whenever $\vec{x}$ and $\vec{y}$ are part of the objects to be ranked, $\vec{x}$ tends to occupy a high position.

\citet{volkovs2009} introduce the algorithm \textsc{BoltzRank}, which learns a
combination of pairwise and individual scoring functions, thus falling under
our category of \ac{FETA} approaches.


\subsection{\acl{FATE}}\label{ssub:af}


To deal with the problem of contexts of variable size, our previous approach was to decompose the context into sub-contexts of a fixed size, evaluate an object $\vec{x}$ in each of the sub-contexts, and then aggregate these evaluations into an overall assessment. An alternative to this ``first evaluate then aggregate'' strategy, and in a sense contrariwise approach, consists of first aggregating the context into a representation of fixed size, and then evaluating the object $\vec{x}$ in this ``super-context''.

More specifically, consider a ranking task $Q$. To evaluate an object $\vec{x}$ in the context $C(\vec{x}) = Q \setminus \{ \vec{x} \}$, the ``first aggregate then evaluate'' (FATE) strategy first computes a representative for the context:
\begin{equation}\label{eq:rep}
  \mu_{C(\vec{x})} = \frac{1}{|C(\vec{x})|} \sum_{\vec{y} \in C(\vec{x})} \phi(\vec{y}) \, ,
\end{equation}
where $\phi \from \cX \to \cZ$ maps each object $\vec{y}$ to
an $m$-dimensional embedding space $\cZ \subseteq \IR^m$. The evaluation itself is then realized by a context-dependent utility function $U:\, \cX \times \cZ \fromto \mathbb{R}$, so that we eventually obtain a ranking 
\begin{equation}
  \rho_{\text{\acs{FATE}}}(Q) = \argsort_{\vec{x} \in Q}\
  U \bigl( \vec{x}, \mu_{C(\vec{x})} \bigr) \, .
\end{equation}
A computationally more efficient variant of this approach is obtained by including an object $\vec{x}$ in its own context, i.e., by setting $C(\vec{x}) = Q$ for all $\vec{x} \in Q$. In this case, the aggregation \eqref{eq:rep} only needs to be computed once. 
Note, that this approach bears resemblance to the recent work by
\citet{zaheer2017} on dealing with set-valued inputs and the general approach
proposed by \citet{ravan2017} on encoding equivariance with respect to
group operations.


\section{Neural Architectures}

In this section, we propose realizations of the FETA and FATE approaches in terms of neural network architectures \bordanet and \newnet, respectively. 
Our design goals for both neural networks are twofold.
First, they should be end-to-end trainable on any differentiable listwise
ranking loss function.
Second, the architectures should be able to generalize beyond the \problem sizes encountered in the training data,
since in practice it is unreasonable to expect all rankings to be of similar
size.
Our focus is on optimizing the 0/1-ranking loss, for which we introduced a
suitable differentiable surrogate loss called the \emph{hinge ranking loss}
which is described in the Appendix.
However, it is also possible to substitute it with any other differentiable loss
function.
\subsection{\bordanet Architecture}
\ac{FETA} as outlined above requires the binary predicate $U_1$ to be given.
In \bordanet, learning this predicate is accomplished by means of a deep neural
network architecture.
More specifically, we make use of the CmpNN architecture
\citep{Rigutini2011,Huybrechts2016}.
\begin{figure}[tb]
  \centering
  \includegraphics[width=0.8\linewidth]{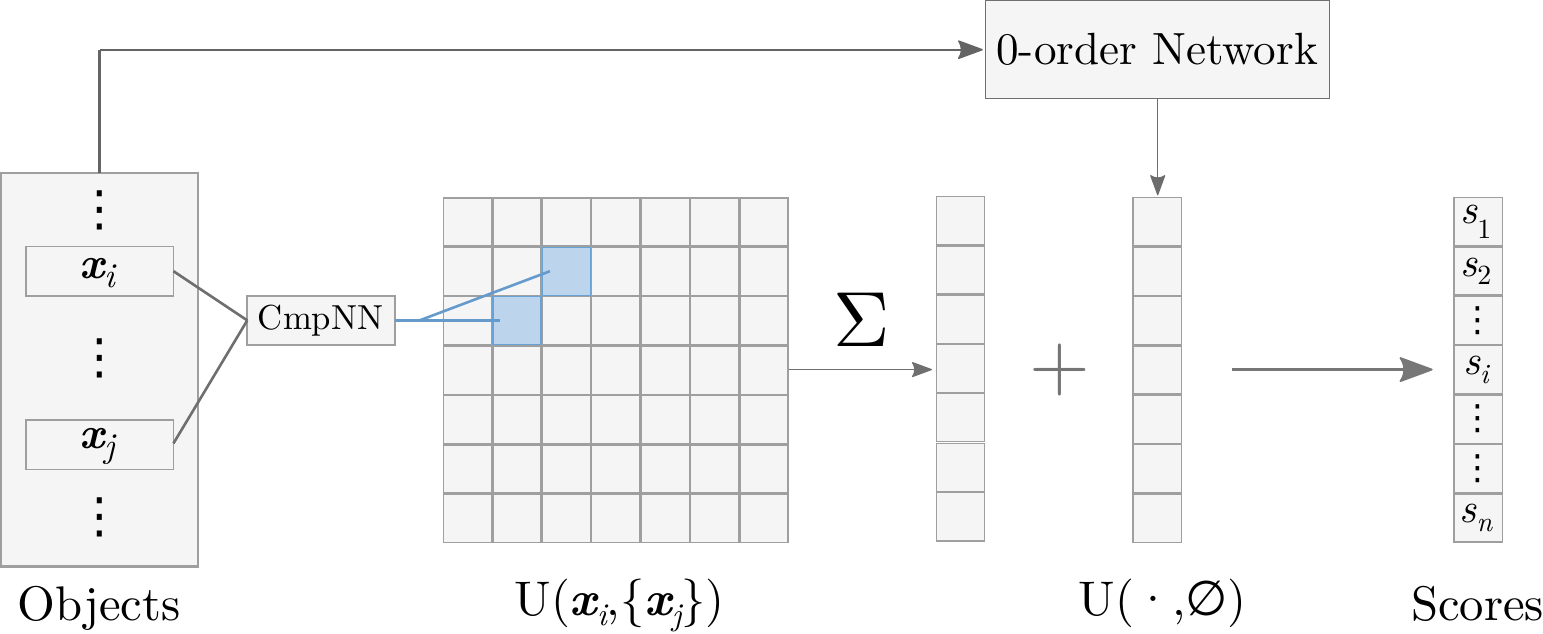}
  \caption{The \bordanet architecture implementing the \ac{FETA}
  decomposition.}
  \label{fig:bordanet}
\end{figure}
In our \bordanet architecture (shown in Figure~\ref{fig:bordanet}),
we evaluate the CmpNN network on all pairs
$(\vec{x}_i, \vec{x}_j)$ of objects in the \problem $Q$ and build up a pairwise
relation $R = (r_{i,j})$, where $r_{i,j} = U_1(\vec{x}_i , \vec{x}_j)$. Using the notation of \citet{Rigutini2011}, this relation is defined as follows:
\begin{equation}
  r_{i,j} = \begin{cases}
    N_{+}([\vec{x}_i, \vec{x}_j]) & \text{if } i < j\\
    N_{-}([\vec{x}_i, \vec{x}_j]) & \text{otherwise}\\
  \end{cases}
\end{equation}
This step is highlighted in blue in Figure~\ref{fig:bordanet}.
Then, each row of the relation $R$ is summed up to obtain a score $\bar{U}_1(\vec{x}, Q)$ for each object $\vec{x}_i \in Q$.
Each $\vec{x}_i$ is also passed through a $0$th-order network that directly
outputs latent utilities $U_0(\vec{x}_i,Q)$.
Here, we use a densely connected, deep neural network with one output unit.
The final score for object $\vec{x}_i$ is then given by
$U(\vec{x}_i,Q) = U_0(\vec{x}_i,Q) + \bar{U}_1(\vec{x},Q)$.

The training complexity of \bordanet is $\mathcal{O}\left(N d q^2\right)$, where $N$ denotes
the number of rankings, $d$ is the number of features per object, and $q$ is an upper bound on the number of objects in each ranking.
For a new ranking task $Q$ (note that we can predict the ranking for any
task size)
the prediction time is in $\mathcal{O}\left(d |Q|^2\right)$.



\subsection{\newnet Architecture}

The \newnet architecture is depicted in Figure~\ref{fig:newnet}.
Inputs are the $n$ objects of the \problem
$Q = \{\vec{x}_1, \dots, \vec{x}_n\}$ (shown in green).
Each object is independently passed through a deep, densely connected embedding
layer (shown in blue).
The embedding layer approximates the function $\phi$ in \eqref{eq:rep}, where, for reasons of computational efficiency, we assume objects to be part of their context (i.e., $C(\vec{x}) = Q$ for all $\vec{x} \in Q$).
Note that we employ weight sharing, i.e., the same embedding is used for
each object.
Then, the representative $\mu_{C(\vec{x})} = \mu_Q$ for the context is computed by averaging the
representations of each object.
To calculate the score $U(\vec{x},Q)$ for an object $\vec{x}_i$, the feature vector
is concatenated with $\mu_Q$ to form the input to the joint hidden layer 
(here depicted in orange).


\begin{figure}[tb]
  \centering
  \includegraphics[width=0.8\linewidth]{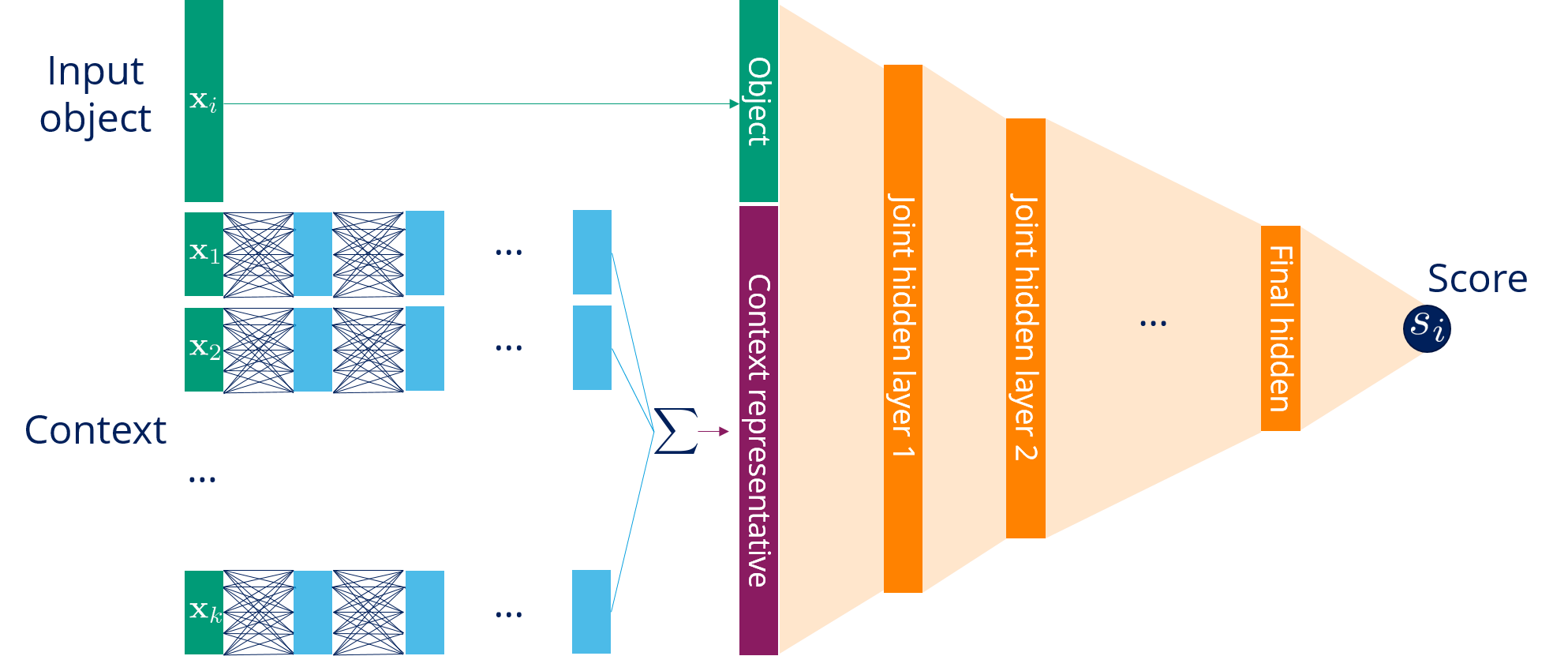}
  \caption{The \newnet architecture implementing the \ac{FATE} decomposition.
  Here we show the score head for object $\vec{x}_i$.}
  \label{fig:newnet}
\end{figure}

The training complexity of \newnet is $\mathcal{O}\left(N d q^2\right)$,
where $N$ denotes the number of rankings,
$d$ is the number of features per object, and $q$ is an upper bound on the
number of objects in each ranking.
For a new query $Q$ (note that we can predict the ranking for any
query size) the prediction can be done in $\mathcal{O}(d |Q|)$ time
(i.e., \emph{linear} in the number of objects).
This is because, if objects are part of their context, the representative $\mu_Q$
has to be computed only once for the forward pass.
This makes the \newnet architecture more efficient to use than \bordanet.


\section{Empirical Evaluation}
\label{sec:empirical_evaluation}

In order to empirically evaluate our \newnet and \bordanet architectures, we make use of synthetic and real-world data.
We mainly address the following questions:
Are the architectures suitable for learning \context-dependent ranking functions, and how do the approaches \ac{FETA} and \ac{FATE} compare with each other?
Can the representation learned on one query size be generalized
    to arbitrary sizes of \problems?

For typical real-world datasets such as OHSUMED and LETOR \citep{letor}, \context-dependence is
difficult to ascertain.
For the evaluation of \context-dependent ranking models, we therefore propose two
new challenging benchmark problems, which are both inspired by real-world problems: the \emph{medoid} and the \emph{hypervolume}
problem. 
Besides, we also analyze a real-world dataset related to the problem of depth estimation in images. 
As baselines to compare with, we selected representative algorithms for three important classes of ranking methods: Expected rank regression (ERR) \citep{kami05} as a representative for so-called pointwise ranking algorithms, \ranksvm \citep{RankSVM} as a state-of-the-art pairwise ranking model, and
deep versions of \ranknet \citep{Burges2005,burges2010,Tesauro1989} and
\listnet \citep{Cao2007,Luo2015a}, which represent the family of deep latent
utility models.

All experiments are implemented in Python, and the code is publicly
available%
\ifcameraready
\footnote{\url{https://github.com/kiudee/cs-ranking}}.
\fi
The hyperparameters of each algorithm were tuned with scikit-optimize
\citep{skopt} using nested cross-validation.
We evaluate the algorithms in terms of 
0/1-accuracy  $\on{d}_{\text{ACC}}(\pi, \tau_{\vec{s}}) = \indic{\pi = \tau_{\vec{s}}}$,
0/1-ranking accuracy $\on{d}_{\text{RA}}(\pi, \vec{s}) = 1 - \on{d}_{\text{RL}}(\pi, \vec{s})$, and Spearman rank correlation $ \on{d}_{\text{Spear}}(\pi, \tau_{\vec{s}}) = 1 - 
    6 \sum_{i=0}^n (\pi(i)-\tau_{\vec{s}}(i))^2)/(n (n^2 - 1))$, where $\tau_{\vec{s}}$ is the ranking induced by the predicted score vector $\vec{s}$.
All implementation details for the experiments are listed in the Appendix.

\subsection{The Medoid Problem}\label{sub:medoid}

Our first synthetic problem is called the medoid problem.
The goal of the algorithms is to sort a set $Q$ of randomly generated points
in $\mathbb{R}^2$ based on their distance to the medoid of $Q$.
This problem is inspired by the setting of similarity learning, where the
goal is to learn a similarity function from triplets of objects \citep{Wang2014}.
The rankings produced by this procedure take the distance to each point in the
\problem into account.
Thus the medoid and subsequently the resulting ranking are sensitive to changes
of the points in the \problem.

For the experiment, we generate \num{100000} sets of \num{5} random points and
determine the rankings as described above.
The instances are split into $\SI{10}{\percent}$ training and $\SI{90}{\percent}$
test data.
This is repeated \num{10} times to get an estimate of the variation across
datasets.

\subsection{The Hypervolume Problem}

In multi-objective optimization, the goal is to find the set of objects that
are non-dominated by any other object in terms of their fitness (solution quality).
The set of all non-dominated objects is called the Pareto-set.
Multi-objective evolutionary algorithms (MOEAs) approximate the Pareto-set by iteratively improving a population of objects.
During optimization, it is not only important to improve the population's fitness, but also to preserve its diversity.
This allows the population to cover the complete Pareto-front.

The \emph{hypervolume} is a set measure, which computes the volume dominated by a
given set of objects. This very naturally encodes both dominance as well as diversity, which makes it a popular fitness criterion \citep{bader2010}.
Usually, we are also interested in the \emph{contribution} of each object/point on the
Pareto-front to the hypervolume.
\citet{bringmann12} proved that computing exact hypervolume contributions is
\sharpP-hard and \NP-hard to approximate.

Our idea is to convert this problem into a challenging, \context-dependent
ranking problem.
The input for the learner is the sets of points on the Pareto-front, and
the target is a ranking of these points based on their contribution to the
hypervolume.
It is apparent that a learner given only a set of data points as input, needs
to take all of the points into account to establish an accurate ranking.

Similar to the Medoid dataset, we generate \num{300000} sets of \num{5}
random points and determine the rankings as described.
The instances are split into $1/3$ training (\num{100000}) and $2/3$ test data
(\num{200000}).
We repeat this \num{5} times to get an estimate of the variation.

\subsection{Image Region Depth Estimation}
As a real-world case study, we tackle the problem of
relative depth estimation of regions in monocular images.
\citet{Ewerth2017} motivate the formalization of this task as an object ranking problem and construct an object ranking dataset on the basis of the Make3D dataset, which consists of 534 photos with an original image resolution of $1704 \times 2272$  \citep{Make3D}:
The images are segmented into $3355$ (i.e., $61 \times 55$) super pixels and
different feature sets are extracted for each super pixel.
We use the feature set \emph{Basic}, where basic monocular depth clues are
available for each super pixel: linear perspective, atmospheric perspective, texture gradients, occlusion of objects, usual size of objects, relative height, relative size, distribution of light. These depth clues suggest that context-dependence could be a relevant issue in depth estimation.
The ground truth rankings are constructed by ordering the super pixels based
on their absolute depth.

Since the size of the rankings is too large for most of the approaches, we
sample subrankings of size $172$ for training. 
\bordanet in addition also samples several subrankings of size $5$ during the
training process.
Predictions are always the complete rankings of size $3355$,
which are compared with the ground truth rankings using 0/1-ranking accuracy.
Super pixels with a distance of more than $\SI{80}{\meter}$ are treated as
tied, because this exceeds the range of the sensor.



\subsection{Results and Discussion}


\begin{table*}[t]
  
  \centering
  \caption{Mean and standard deviation of the losses on the medoid (top) and hypervolume (below) data
    (measured across 10 outer cross validation folds). Best entry for each
    loss marked in bold.}
  \sisetup{
    table-align-uncertainty=true,
    separate-uncertainty=true,
  }
  \renewrobustcmd{\bfseries}{\fontseries{b}\selectfont}
  \renewrobustcmd{\boldmath}{}
  \resizebox{\textwidth}{!}{
  \begin{tabular}{
    l
    S[table-format=0.3(3),detect-weight,mode=text]
    S[table-format=0.3(3),detect-weight,mode=text]
    S[table-format=0.3(3),detect-weight,mode=text]
    S[table-format=0.3(3),detect-weight,mode=text]
    S[table-format=0.3(3),detect-weight,mode=text]
    S[table-format=0.3(3),detect-weight,mode=text]}
  \toprule
  Ranker & \err & \ranksvm  & \ranknet & \listnet &  \bordanet & \newnet  \\
  \midrule
  $\on{d}_{\text{Spear}}$ &   -.001 +- .002 &  .000+-.002  &  .417+-.001 & 0.361(3) & 0.594(32) & \bfseries 0.861+-0.007 \\
  $\on{d}_{\text{RA}}$ & .500+-0.001  & .500+-0.001 & .682+-0.001 & .681(2) & 0.759(14) & \bfseries 0.901+-0.004\\
  $\on{d}_{\text{ACC}}$ & 0.008+-0.001 & 0.008+-0.001 & 0.088+-0.001 & .087(1) & 0.088+-0.001  & \bfseries 0.443+-0.016\\
   \midrule
   $\on{d}_{\text{Spear}}$ & 0.001(2)  & -0.001(2)   &  0.419(1) & .418(1) &  0.682(5)   & \bfseries 0.894(4)  \\
   $\on{d}_{\text{RA}}$ &  0.500(1) & 0.500(1)  & 0.683(1) & .683(0) & 0.802(2) &  \bfseries 0.920(3)  \\
   $\on{d}_{\text{ACC}}$ &  0.008(1) & 0.008(0) & 0.089(1) & .07(39) & 0.192(4) &  \bfseries 0.508(13)  \\ 
  \bottomrule
  \end{tabular}}
  \label{tab:medoidhyper}
\end{table*}

The results on the Medoid and the Hypervolume dataset are shown in
Table~\ref{tab:medoidhyper}.
ERR and \ranksvm completely fail on both tasks, both having a correlation
of 0 with the target rankings.
This can be explained by the fact that both approaches ultimately learn a linear
model, while the problems are highly non-linear.
Being non-linear latent-utility approaches, \ranknet and \listnet are able to improve upon
random guessing and achieve a 0/1-ranking accuracy of around \SI{68}{\percent}.
This result is surprising, considering that both networks establish the final ranking 
by scoring each point independently, i.e., not taking the other points of
the \problem into account. 

Our architectures \bordanet and \newnet are both able to make use of the given
\context provided by the \problem and beat the \context-insensitive
approaches by a wide margin.
With a 0/1-ranking accuracy of more than \SI{90}{\percent},
\newnet  even performs significantly better than \bordanet.
This suggests that the pairwise decomposition (first-order
approximation) is not able to completely capture the
higher order interactions between the objects.

Since the ranking size is fixed for the Medoid and Hypervolume dataset,
we ran additional experiments we varied the size of the rankings during
test time.
The results are shown in Figure~2 of the Appendix.


\begin{table*}[t]
  \centering
  \caption{Results on the test set for the relative depth dataset with Basic
    feature set.}
  \sisetup{
    table-align-uncertainty=true,
    separate-uncertainty=true,
  }
  \renewrobustcmd{\bfseries}{\fontseries{b}\selectfont}
  \renewrobustcmd{\boldmath}{}
  \begin{tabular}{
    l
    S[table-format=3.3(3),detect-weight,mode=text]
    S[table-format=3.3(3),detect-weight,mode=text]
    S[table-format=3.3(3),detect-weight,mode=text]
    S[table-format=3.3(3),detect-weight,mode=text]}
  \toprule
  Ranker &  {Spearman correlation} & {0/1-Ranking accuracy} \\
  \midrule
      \err     & 0.659 & 0.792 \\
      \ranksvm & 0.656 & 0.778 \\
      \ranknet &  0.633 & 0.770 \\
      \listnet &  0.693 &  0.802 \\
      \rankboost \footnotemark &  & 0.806\\
   \midrule
    \bordanet & \bfseries 0.729 & 0.813 \\
    \newnet &  0.721 & \bfseries 0.814 \\
  \bottomrule
  \end{tabular}
  
  \label{tab:depthbasic}
\end{table*}

The results for the relative depth estimation problem are shown in
Table~\ref{tab:depthbasic}.
We additionally report the results obtained by \citet{Ewerth2017} using
Rankboost on the same dataset (using the same split into training and test).
Both our architectures achieve comparably high Spearman correlation and
ranking accuracy, and slightly outperform the competitors.

\section{Conclusion and Future Work}

In this paper, we addressed the novel problem of learning
\context-dependent ranking functions in the setting of object ranking
and, moreover, proposed two general solutions to this problem. These
solutions are based on two principled ways for representing
context-dependent ranking functions that accept ranking problems of any
size as input and guarantee symmetry.  \ac{FETA} (first evaluate then
aggregate) is a first-order approximation to a more
general latent-utility decomposition, which we proved to be flexible
enough to learn any ranking function in the limit.
\ac{FATE} (first aggregate then evaluate) first transforms each object
into an embedding
space and computes a representative of the context by averaging.
Objects are then scored with this representative as a fixed-size context.
\footcitetext{Ewerth2017, Freund2003}

To enable end-to-end optimization of differentiable ranking losses using
these decompositions, we further contribute two new neural network
architectures
called \bordanet and \newnet.
We demonstrate empirically that both architectures are able to learn
context-dependent ranking functions on both synthetic and real-world data.

While \ac{FETA} and \ac{FATE} appear to be natural approaches to
\context-dependent
ranking, and first experimental results are promising, the theoretical
foundation of context-dependent ranking is still weakly developed. One
important question concerns the expressivity of the two representations,
i.\,e., what type of context-effects they are able to capture, and what
class of context-dependent ranking functions they can model. As already
said, a first result could be established in the case of \ac{FETA},
showing that any ranking function on $N$ objects can be modeled by a
decomposition of order $N-1$ (cf.\ supplementary material). Yet, while
this result is theoretically interesting, a quantification of the
expressivity for practically meaningful model classes (i.\,e., $k$th-order
approximations with small $k$) is an open question. Likewise, for \ac{FATE},
there are no results in this direction so far.


\ifcameraready
\section*{Acknowledgements}
This work is part of the Collaborative Research Center ``On-the-Fly Computing'' at Paderborn University, which is supported by the German Research Foundation (DFG).
Calculations leading to the results presented here were performed on 
resources provided by the Paderborn Center for Parallel Computing.
\fi

\nocite{skopt,Ioffe2015,selu,nesterov1983,rooderkerk2011incorporating,huber1983market}

\printbibliography

\section*{Appendix}

\subsection*{Context effects}
  The compromise effect states that the relative utility of an object
  increases by
  adding an extreme option that makes it a compromise in the set of
  alternatives
  \citep{rooderkerk2011incorporating}.
  For instance, consider the set of objects $\{A, B\}$ in
  Figure~\ref{fig:compromise}.
  The ordering of these objects depends on how much the consumer is
  weighing the
  quality and the price of the product.
  If price is the constraint, then the preference order will be $A \succ B$.
  But as soon as there is another extreme option $C$ available, the
  object $B$ becomes a compromise option between the three alternatives.
  The preference relation between $A$ and $B$ gets inverted and turns into
  $B \succ A$.

  Figure~\ref{fig:attraction} illustrates the attraction effect.
  Here, if we add another object $C$ to the
  set of objects $\{A, B\}$, where $C$ is slightly dominated by $B$,
  the relative utility share for object $B$ increases with respect to $A$.
  The major psychological reason is that consumers have a strong
  preference for dominating products \citep{huber1983market}.
  Thus, the preference relation between $A$ and $B$ may again be influenced.

  The similarity or substitution effect is another phenomenon, according
  to which
  the presence of similar objects tends to reduce the overall probability
  of an object to be chosen, as it will divide the loyalty of potential
  consumers
  \citep{huber1983market}.
  In Figure~\ref{fig:similarity}, $B$ and $C$ are two similar objects.
  Consumers who prefer high quality will be divided amongst the two objects,
  resulting in a decrease of the relative utility share of object $B$.
  Again, this may lead to turning a preference $B \succ A$ into $A \succ
  B$, at least on an aggregate (population) level, if preferences are
  defined on the basis of choice probabilities.

  \begin{figure}[hb]
    \centering
    \begin{subfigure}[c]{0.3\linewidth}
      \includegraphics[width=\linewidth]{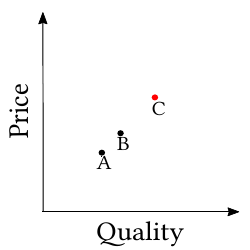}
      \subcaption{Compromise}
      \label{fig:compromise}
    \end{subfigure}
    \begin{subfigure}[c]{0.3\linewidth}
      \includegraphics[width=\linewidth]{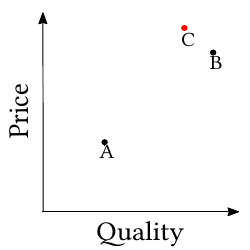}
      \subcaption{Attraction}
      \label{fig:attraction}
    \end{subfigure}
    \begin{subfigure}[c]{0.3\linewidth}
      \includegraphics[width=\linewidth]{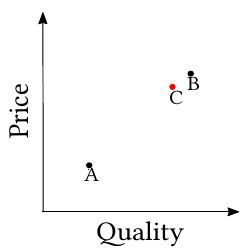}
      \subcaption{Similarity}
      \label{fig:similarity}
    \end{subfigure}
    \caption{Context effects identified in the literature
    \citep{rooderkerk2011incorporating}.}
    \label{fig:context}
  \end{figure}
\subsection*{Expressivity of $K$th order approximation}
  \begin{proposition}\label{prop:bc}
    Let $\cX$ be a set of $N$ objects and $\cQ = 2^{\cX} \setminus \emptyset$
    be the corresponding \problem space.
    Let $\rho \sothat \cQ \to \cR$
    be a ranking function mapping from queries to rankings with $\cR = \bigcup_{n\in [N]} \IS_n$.
    There always exist preference functions
    \[
      U_k \sothat \cX \times \cX^k \to \IR
    \]
    such that the corresponding ranking rule of order $N-1$
    \begin{align}\label{eq:bc}
      \rhobc(C) &= \argsort_{i \in [N]}
      U(\vec{x}_i, C\setminus\{\vec{x}_i\})
      \\
      &= \rho(C)
    \end{align}
    for all \problems $C \in \cQ$.
  \end{proposition}
  \begin{proof}
    Let $C_{-i} = C \setminus \{\vec{x}_i\}$ be the context for the 
    \problem $C$ when scoring object $\vec{x}_i$.
    First, notice that $\rhobc(C) = \rho(C) = \pi$ if and only if
    $U(\pi(1), C_{-1}) > U(\pi(2), C_{-2}) > \dots > U(\pi(n), C_{-n})$ for a given
    \problem $C$ of size $n$.
    Thus $\rhobc(C) = \rho(C)$ for all $C \in \cQ$ if and only if all
    resulting inequalities defined on the scores are satisfied.

    The result can be shown by induction over the size $K \in \IN$ of the
    maximum \problem size (i.e. $K = \max_{C\in \cQ} |C|$) for which
    $\rho$ is defined.
    In addition denote with $\cQ^{(k)}$ the \problem space of \problems with
    $|C| \leq k$ for $C \in \cQ^{(k)}$.
    For the base case of $K=1$ the corresponding rankings are all of
    size $1$, which is why $\rhobc(C) = \rho(C)$ trivially holds for
    all $C\in \cQ$.

    For $K=2$ we have ranking functions defined on pairs of objects.
    Here it suffices to use a preference function $U_1$ with contexts of size $2$.
    Note, that for one fixed \problem $\{\vec{x}_i, \vec{x}_j\}$
    each preference score $U_2(\vec{x}_i, \{\vec{x}_j\})$ only appears
    in one inequality (i.e. the one in which $U_2(\vec{x}_i, \{\vec{x}_j\})$
    and $U_2(\vec{x}_j, \{\vec{x}_i\})$ are compared).
    It follows that we can set
    $U_2(\vec{x}_i, \{\vec{x}_j\}) = \indic{\vec{x}_i \succ_{\rho} \vec{x}_j}$
    for all pairs $i,j \in [N]$ and
    $U_1(\vec{x}_i, \cdot) = 0$ for all $i \in [N]$.

    For the inductive step assume that for all $k \leq K$ the equality
    $\rhobc(C) = \rho(C)$ holds for all $C \in \cQ^{(k)}$.
    Now we consider the step $K \to K+1$.
    We know that 
    preference scores $U(\vec{x}_i, C_{-i})$ with $|C| = K+1$ only occur for
    inequalities defined for rankings of size $K+1$ and further appear only
    in one inequality since we evaluate it only once for each
    object $\vec{x}_i$.

    This time it is not possible to independently set these preference scores as
    we did before,
    since we have to take into account all the summands in equation (2)
    of the main paper.
    We already know by induction hypothesis that there exists a preference function
    $U$ for any $\rho$ defined on $\cQ^{(K)}$ such that $\rhobc(C) = \rho(C)$
    for all $C\in \cQ^{K}$.
    For any \problem $C \in \cQ^{(K)}$ with $|C| = K$ let 
    \begin{align}
      \delta_{\max} = \max_{\vec{x}_i, \vec{x}_j \in X} \biggl|
      \ U(\vec{x}_i, C_{-i}) - U(\vec{x}_j, C_{-j})\biggr|
    \end{align}
    be the maximum score difference using only the existing preference scores.
    We will use $\delta_{\max}$ now as a step size to define the preference scores
    for rankings of size $K+1$.

    Then set the preference scores $U_{K+1}(\cdot, C_{-\cdot})$ for $|C| = K + 1$ as follows:
    \begin{equation}\label{eq:constr}
      U_{K+1}(\vec{x}_i, C_{-i}) = (K+1 - \pi^{-1}(i)) \cdot  (\delta_{\max} + \varepsilon)
    \end{equation}
    where $\pi^{-1} = \rho^{-1}(X \cup \{\vec{x}_i\})$ and $\varepsilon > 0$.
    In other words, we simply set the scores inversely proportional to the
    position of the object in their respective ranking.
    To guarantee, that the preference scores which were defined on $|C| < K$
    do not have an effect, we additionally scale by the step size
    $\delta_{\max} + \varepsilon$.

    It follows that for any $\vec{x}_i, \vec{x}_j$ and any~$C$ with $|C| = K+1$:
    \begin{align}
      &U(\vec{x}_i, C_{-i}) - U(\vec{x}_j, C_{-j}) \\
  &\leq U_{K+1}(\vec{x}_i, C_{-i})- U_{K+1}(\vec{x}_j, C_{-j}) + \delta_{\max}\\
  &= (\delta_{\max} + \varepsilon) (\pi^{-1}(j) - \pi^{-1}(i)) + \delta_{\max}\\
  &= \delta_{\max} (\pi^{-1}(j) - \pi^{-1}(i) + 1) + \varepsilon (\pi^{-1}(j) - \pi^{-1}(i))
    \end{align}
    Since $\delta_{\max} > 0$ this equation is $ > 0$ if and only if 
    $\pi^{-1}(j) > \pi^{-1}(i)$ and therefore if $\vec{x}_i \succ \vec{x}_j$
    and $< 0$ otherwise.

    We can conclude that we can obtain all possible rankings of size $K+1$
    using this construction. Thus, the statement follows.
  \end{proof}

\subsection*{Loss Functions}\label{ssub:losses}
  A key advantage of the above architectures is that they are fully differentiable,
  allowing us to use any differentiable loss function $\on{\ell}$. In our case, a loss is supposed to compare a ground-truth ranking $\pi$ for a task $Q = \{ \vec{x}_1, \ldots , \vec{x}_n \}$ with a vector $\vec{s}=(s_1, \ldots , s_n)$ of scores predicted for the objects in $Q$. Thus, the loss is of the form $\on{\ell}(\pi, \vec{s})$.

  Unfortunately, many interesting ranking losses, such as the (normalized) \emph{0/1-ranking loss} 
  \begin{align}
    \on{d}_{\text{RL}}(\pi, \vec{s}) = 
      \frac{2}{n(n-1)} \qquad
        \sum_{\mathclap{\substack{(i,j): \pi(i) < \pi(j)}}}
           \Bigl(\indic{s_i < s_j}
         + \frac{1}{2} \indic{s_i = s_j}\Bigr)
         \label{eq:rankloss}
  \end{align}
  or the popular nDCG, are not differentiable.
  Yet, just like in the binary classification setting, we can define a
  \emph{surrogate loss function} that upper bounds the true binary ranking loss,
  is differentiable almost everywhere, and ideally even convex.
  We propose to use the \emph{hinge ranking loss}:
  \begin{align}
    \on{\ell}_{\text{HL}}(\pi, \vec{s}) = 
      \frac{2}{n(n-1)}
        \ \sum_{\mathclap{\substack{(i,j): \\\pi(i) < \pi(j)}}}
        \ \max \big(1 + s_i - s_j, 0 \big)
        \label{eq:hingerl}
  \end{align}
  It is convex and has a constant subgradient with respect to the
  individual scores. Another choice for a differentiable loss function is the
  Plackett-Luce (PL) loss:
  \begin{equation}\label{eq:plloss}
    \ell_{\text{PL}}(\pi, \vec{s}) = 
    \sum_{i=1}^{n-1} \log \biggl( \sum_{j=i}^n \exp\left(s_{\pi^{-1}(j)}\right) \biggr)
    - s_{\pi^{-1}(i)} \enspace ,
  \end{equation}
  which corresponds to the negative logarithm of the PL-probability to observe
  $\pi$ given parameters $\vec{s}$.
  The networks can then be trained by gradient descent and backpropagating the
  loss through the network.

\subsection*{Experimental details}


All experiments are implemented in Python, and the code is publicly
available%
\ifcameraready
\footnote{\url{https://github.com/kiudee/cs-ranking}}.
\fi
The hyperparameters of each algorithm were tuned with scikit-optimize
\citep{skopt} using nested cross-validation.
For all neural network models, we make use of the following techniques:
We use either ReLU non-linearities + batch normalization \citep{Ioffe2015} or
    SELU non-linearities \citep{selu} for each hidden layer.
For regularization, both $L_1$ and $L_2$ penalties are applied.
 For optimization, stochastic gradient descent with Nesterov momentum
    \citep{nesterov1983} is used.
\listnet has an additional parameter $k$, which specifies the size of the
top-$k$ rankings used for training.
We set this parameter to $3$ in all experiments.

We evaluate the algorithms in terms of 
0/1-accuracy  $\on{d}_{\text{ACC}}(\pi, \tau_{\vec{s}}) = \indic{\pi = \tau_{\vec{s}}}$,
0/1-ranking accuracy $\on{d}_{\text{RA}}(\pi, \vec{s}) = 1 - \on{d}_{\text{RL}}(\pi, \vec{s})$, and Spearman rank correlation $ \on{d}_{\text{Spear}}(\pi, \tau_{\vec{s}}) = 1 - 
    6 \sum_{i=0}^n (\pi(i)-\tau_{\vec{s}}(i))^2)/(n (n^2 - 1))$, where $\tau_{\vec{s}}$ is the ranking induced by the predicted score vector $\vec{s}$.

  \subsubsection*{The Medoid Problem}
    Specifically, we generate
    $\mathcal{D} = \{(Q_1, \pi_1), (Q_2, \pi_2), \dotsc, (Q_N, \pi_N)\}$
    as follows: 
    \begin{enumerate}
      \item Generate data points $Q_i = \{ \vec{x}_{i,1}, \dotsc, \vec{x}_{i,n} \} \subseteq [0, 1]^d$ for the \problems uniformly at random.
      \item Construct the corresponding rankings $\pi_1, \dotsc, \pi_N$ as follows:
        For all $1\leq i\leq N$,
        \begin{enumerate}
          \item compute the medoid
            \[\vec{x}_i^* = \argmin_{\vec{x} \in Q_i} \frac{1}{n} \sum_{j=1}^n \norm{\vec{x} - \vec{x}_{i,j} } \, ,
            \]
          \item compute the ranking 
            \begin{equation}
              \pi_i = \argsort_{j \in [n]} \,  - \norm{\vec{x}_i^* - \vec{x}_{i,j}} \, .
            \end{equation}
      \end{enumerate}
    \end{enumerate}
  \subsubsection*{The Hypervolume Problem}
    The input for the learners are the sets of points on the Pareto-front, and
    the target is a ranking of these points based on their contribution to the
    hypervolume.
    Data generations is done as follows:
    \begin{enumerate}
         \item Generate points for the \problems uniformly on the negative
    surface of
             the unit sphere
             $Q_i = \{\vec{x}_{i,1}, \dotsc, \vec{x}_{i,n} \} \subseteq
             \{\vec{x} \in \IR^d \mid \norm{\vec{x}} = 1, \forall 1\leq
    i\leq d: x_i \leq 0\}$.

         \item Construct the corresponding rankings $\pi_1, \dotsc, \pi_N$
    as follows:
             For all $1\leq i\leq N$,
             \begin{enumerate}
                 \item compute the contributions $\Delta_j$ of each object
    $\vec{x}_{i,j}$
                 to the hypervolume\footnote{We use the PyGMO library to
    compute exact contributions.}
                 of the \problem $Q_i$, i.e.,
                     \[
                         \Delta_{i,j} = \hyp(Q_i) - \hyp(Q_i \setminus
    \{\vec{x_{i,j}}\}) \, ,
                     \]
                 \item and then the ranking
                     \begin{equation}
                         \pi_i = \argsort_{j \in [n]} \; \{-\Delta_{i,j}
    \mid \vec{x}_{i,j} \in Q_i\} \, .
                     \end{equation}
         \end{enumerate}
    \end{enumerate}

\subsection*{Generalization across ranking sizes}\label{sec:otherresults}
  \begin{figure}[htbp]
    \centering  
    \includegraphics[width=0.9\linewidth]{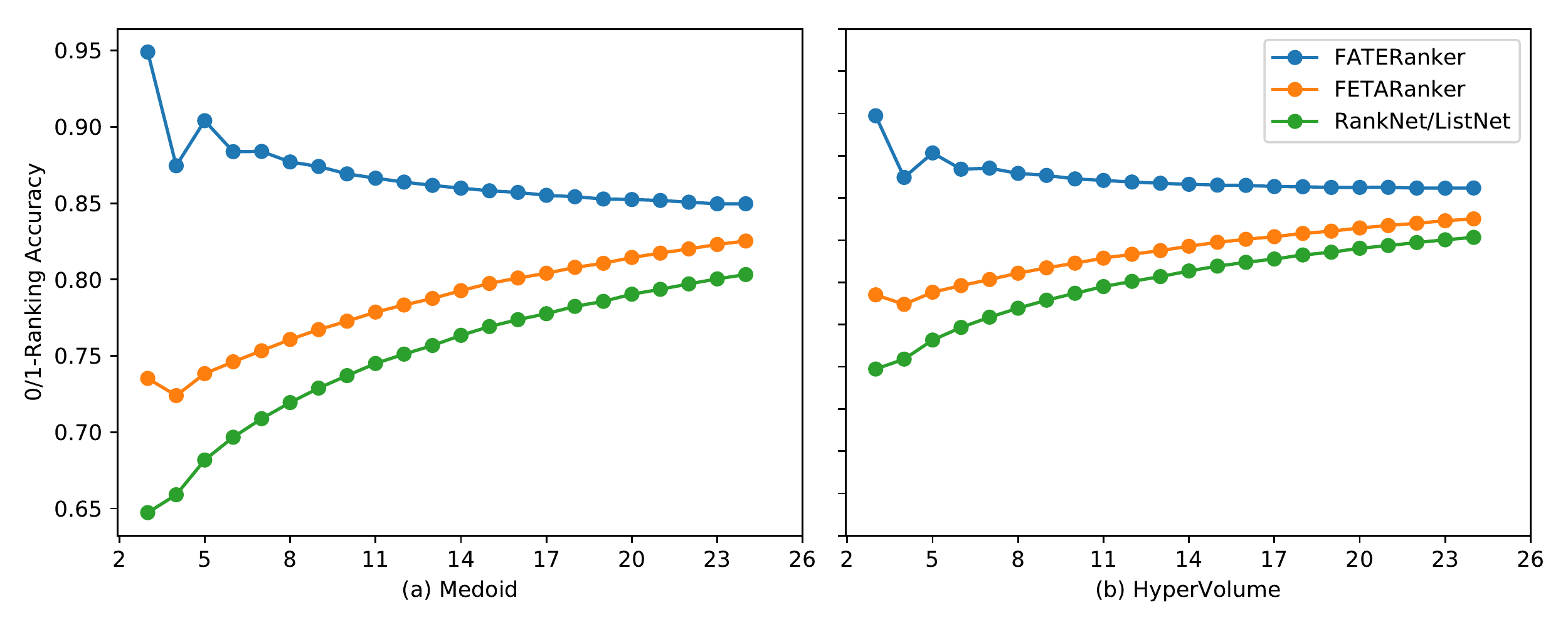}
    \caption{0/1-Ranking accuracy of \newnet trained on rankings of size 5, then
    predicting on \problems of a different size.}
    \label{fig:generalization}
  \end{figure}
  Given that our approaches \newnet and \bordanet outperform the other approaches,
  we are interested in how well all of the approaches generalize to unseen \problem sizes.
  To this end, we apply them on the Medoid and Hypervolume dataset with 5 objects as
  \problem size, and then test it on \problem sizes between 3 and 24.
  The results are shown in Figure~\ref{fig:generalization}.
  Expected rank regression and RankSVM did not achieve better results than random guessing,
  which is why we removed them from the plot.
  It is interesting to note, that for models which model the latent utility
  (i.\,e. \bordanet and \ranknet), the accuracy improves with increasing \problem size.
  This hints at the context-dependency vanishing with increasing number of 
  objects, since they densely populate the space.
  For \newnet, increasing the size of the \problem apparently leads to a 
  slightly lower ranking accuracy.
  Since the representative $\mu_Q$ is computed as the
  average of the object embeddings, this behavior is to be expected
  when the ranking function behaves similarly for different \problem sizes.
\end{document}